\documentclass{article}
\usepackage{arxiv}

\usepackage[utf8]{inputenc} 
\usepackage[T1]{fontenc}    
\usepackage{hyperref}       
\usepackage{url}            
\usepackage{booktabs}       
\usepackage{amsfonts}       
\usepackage{nicefrac}       
\usepackage{microtype}      
\usepackage[dvipsnames]{xcolor}         
\usepackage{wrapfig}
\usepackage{natbib}

\usepackage{amsmath}
\usepackage{amsthm}
\usepackage{amssymb}
\usepackage{amsfonts}
\usepackage[capitalize]{cleveref}
\usepackage{graphicx}
\usepackage{bm}

\usepackage{thm-restate}

\newcommand{\R}{\mathbb{R}}
\newcommand{\Z}{\mathbb{Z}}
\newcommand{\defeq}{:=}

\newcommand{\ev}{\mathbb{E}}
\newcommand{\prob}{\mathbb{P}}

\DeclareMathOperator*{\argmin}{arg\,min}

\DeclareMathOperator{\tr}{tr}

\newcommand{\transpose}{\mathsf{T}}

\newtheorem{theorem}{Theorem}[section]
\newtheorem{corollary}[theorem]{Corollary}
\newtheorem{lemma}[theorem]{Lemma}
\newtheorem{remark}[theorem]{Remark}
\newtheorem{proposition}[theorem]{Proposition}

\newcommand{\vecn}{\bm{v}}
\newcommand{\vecq}{\bm{q}}
\newcommand{\vecr}{\bm{r}}
\newcommand{\vecu}{\bm{u}}
\newcommand{\vecv}{\bm{v}}
\newcommand{\vecw}{\bm{w}}
\newcommand{\vecx}{\bm{x}}
\newcommand{\vecz}{\bm{z}}

\newcommand{\wup}[1]{\vecw^{(#1)}}
\newcommand{\vup}[1]{\vecv^{(#1)}}
\newcommand{\vecui}[1]{\vecu^{(#1)}}
\newcommand{\xsamp}[1]{\vecx^{(#1)}}
\newcommand{\A}{\mathcal{A}}

\newcommand{\stdbasis}[1]{\bm{e}^{(#1)}}

\newcommand{\matA}{\mathbf{A}}
\newcommand{\matI}{\mathbf{I}}
\newcommand{\matP}{\mathbf{P}}
\newcommand{\matS}{\mathbf{S}}
\newcommand{\matU}{\mathbf{U}}

\newcommand{\matW}{\mathbf{W}}
\newcommand{\matX}{\mathbf{X}}
\newcommand{\matSigma}{\mathbf{\Sigma}}
\newcommand{\proj}[1]{\matP_{{>#1}}^\perp}

\title{Generalization behavior of OPTQ and the role of regularization}

\author{%
  Erin George\\
  Department of Mathematics\\
  University of California, San Diego\\
  La Jolla, CA 92093\\
  \texttt{e2george@ucsd.edu} \\
  \And
  Rayan Saab \\
  Department of Mathematics and Halıcıoğlu Data Science Institute\\
  University of California, San Diego \\
  La Jolla, CA 92093 \\
  \texttt{rsaab@ucsd.edu} \\
}

\begin{document}

\maketitle

\begin{abstract}
Large neural networks can be compressed by rounding or ``quantizing'' their weights to numbers that admit representations with fewer bits.  One algorithm for quantization, OPTQ, progressively quantizes the weights of a neural network so that the squared quantization error on a specified calibration dataset is as small as possible.  We study the performance of OPTQ and a variant algorithm, stochastic OPTQ, in a generalization setting and derive bounds for the expected squared error accrued by the algorithm when a test point is drawn from a fixed distribution.  We prove two  results.  One result relates the generalization error to the error on a calibration dataset comprising independent samples from the same distribution as the test distribution.  The other result bounds the generalization error of stochastic OPTQ for all sufficiently nice distributions, regardless of the calibration dataset.  In both of these results, the regularization term $\lambda$ plays an important role.  We use insights from these results to make a new recommendation for the choice of $\lambda$ and see that this choice of $\lambda$ preforms favorably in experiments when compared to prior recommendations in the literature.
\end{abstract}

\section{Introduction}

Modern neural networks typically feature massive amounts of parameters.  Early large language models (LLMs) used billions of parameters, such as GPT-3~\citep{NEURIPS2020_1457c0d6} (175B) and LLaMa~\citep{touvron2023llamaopenefficientfoundation} (7--65B).  Recent AI models are even larger, with Deepseek-V4-Pro~\citep{deepseekai2026deepseekv4highlyefficientmilliontoken} having 1.6 trillion parameters for example.  In turn, storing and using these models can require substantial amount of memory and compute time.  Several paradigms have emerged recently to reduce these costs, including pruning, low-rank factorization, and knowledge distillation~\citep{dantas_comprehensive_2024}.  One technique of particular interest is to store the parameters of the neural network in representations that use fewer bits.  This process is called \emph{quantization}.

There are two major types of methods for quantization~\citep{app14177445}.  One type, called quantization-aware training (QAT) modifies the optimization of neural networks to enable efficient quantization~\citep{jacob2018quantization,zhang2025survey}.  The other major type of quantization algorithms is post-training quantization (PTQ).  Methods for PTQ work on an already trained neural network and attempt to quantize the network with as minimal an impact on its performance as possible.  In comparison to QAT methods, PTQ methods typically require less data and computational resources but are usually not able to achieve the same level of accuracy~\citep{app14177445}.  A popular PTQ method that has emerged recently is OPTQ, also known as GPTQ~\citep{frantar2023optq}.  OPTQ is able to successfully quantize many LLMs to 4-bit weights~\citep{frantar2023optq,jin-etal-2024-comprehensive} and is a component of more sophisticated quantization algorithms, such as QuIP~\citep{NEURIPS2023_0df38cd1} and Qronos~\citep{zhang2025qronos}.

A line of work has been developed to analyze OPTQ from a theoretical perspective. \citet{zhang2025provable} bound overall and coordinate-wise errors for OPTQ on the calibration dataset used to quantize a neural network model.  \citet{birnick2026latticegeometryneuralnetwork,chen2026optqbabai} establish connections between OPTQ and Babai's algorithm for the closest vector problem.  Despite this progress, little is known about how OPTQ is able to generalize outside of the calibration dataset.  This paper aims to address this problem by deriving explicit bounds for the generalization error of OPTQ in a variety of settings.  We are motivated by the perspective that calibration error (the error on the dataset used to guide the quantization) and generalization error (the expected error on an unseen datapoint) are distinct quantities for data-dependent PTQ methods.  In particular, OPTQ can make the error small on the calibration data by exploiting the empirical geometry of the sampled activations, but this empirical geometry may contain spurious correlations when the calibration set is small.  Regularization plays an important role in our analysis, mainly to prevent this form of overfitting.  The contributions of the paper should therefore be viewed in two layers.  First, we prove a regularized empirical-to-population comparison inequality that is not intrinsically tied to OPTQ and can be applied to any data-dependent quantizer whose regularized empirical error is controlled.  Second, we specialize this comparison to OPTQ using its update structure, and we prove a complementary positive-semidefinite quadratic-form bound for stochastic OPTQ.

\subsection{Notation}

Throughout this paper, we use the following notation.  For a real number $x$, the quantity $\lfloor x \rfloor$ is the floor function, which maps its argument down to the largest interger not exceeding $x$.  The abbreviation ``i.i.d.'' stands for independent and identically distributed.  For a random variable $X$, we use $\ev_X$ and $\prob_X$ to denote the expectation and probability over only the randomness in $X$.  Boldfaced lowercase letters (e.g., $\vecu$) will denote vectors and boldfaced uppercase letters (e.g., $\matA$) will denote matrices.  For a vector $\vecu$, we use $\vecu_i$ to denote value of the $i$-th coordinate of $\vecu$.  For a matrix $\matA$, we use $\matA_i$ to denote the $i$-th column of $\matA$ (as a vector).  We extend this indexing notation to allow taking multiple coordinates or rows at a time, for example we use $\matA_{>i}$ to denote the submatrix formed by considering only the rows with index larger than $i$.  The vector $\stdbasis{i}$ is the $i$-th standard basis vector, which takes the value $1$ and the $i$-th coordinate and $0$ in all others.  The matrix $\matI$ is the identity matrix.  We denote the psuedoinverse of a matrix $\matA$ to be $\matA^\dagger$.  For vectors, $\|\cdot\|$ refers to the $\ell^2$-norm and $\|\cdot\|_{\infty}$ refers to the $\ell^\infty$-norm.  For matrices, $\|\cdot\|_{\textrm{F}}$ refers to the Frobenius norm and $\|\cdot\|_{\textrm{op}}$ refers to the $\ell^2\to\ell^2$ operator norm.  The function $\tr(\cdot)$ denotes the trace of its matrix argument.  We use $\succeq$ to denote the Loewner order on positive semi definite matrices.

\section{Setting}

Starting from a weight vector $\vecw \in \R^N$, the goal of post-training quantization algorithms is to obtain a quantized vector $\vecq \in \A^N$ that approximates $\vecw$.  Here, $\A$ represents the quantization alphabet, which is a discrete set of real numbers.  As an example, $\A$ can be the following set
\begin{equation}\A = \{\delta k : k = -2^b, -2^b + 1, \ldots, 2^b - 2, 2^b -1\}\label{eq:alphabet_delta_b}\end{equation}
which allows us to represent $\vecq$ as a vector of $(b+1)$-bit integers scaled by some value $\delta > 0$.  Quantizing the real-valued vector $\vecw$ will then accrue errors in two ways: by rounding and by clipping.   In this paper, we will only be interested in the error caused by the rounding process, so we will work with the alphabet $\A = \{\delta k : k \in \Z\}$.  This is a model for an alphabet of the form given by \cref{eq:alphabet_delta_b} where $b$ is chosen sufficiently large so that there is no clipping performed.

\subsection{OPTQ}

OPTQ is a data-dependent post-training quantization algorithm introduced by \citet{frantar2023optq}.  Given a weight vector $\vecw\in\R^N$ and a calibration matrix $\matX\in\R^{m\times N}$, OPTQ constructs a quantized vector $\vecq$ by sequentially quantizing the coordinates of $\vecw$ while attempting to keep the calibration error
\(
    \|\matX\vecw-\matX\vecq\|^2
\)
small.  The key feature of OPTQ is that, after quantizing coordinate $i$, the algorithm updates the remaining unquantized coordinates $j>i$ to compensate for the error introduced at coordinate $i$.

We recall a mathematical formulation of this OPTQ update structure \citep{frantar2023optq,zhang2025qronos}, first focusing on the unregularized version. 
OPTQ maintains a working vector, initialized as
\begin{equation}
    \vecw^{(1)}=\vecw.
\end{equation}
At step $i$, in the standard (deterministic) version of OPTQ, the current value of the $i$-th coordinate is quantized by setting
\begin{equation}
    \vecq_i=\mathsf{Q}_\delta\!\left(\vecw_i^{(i)}\right),
\end{equation}
where $\mathsf{Q}_\delta$ denotes nearest-neighbor quantization onto the alphabet $\delta\mathbb{Z}$, i.e.,
\begin{equation}
    \mathsf{Q}_\delta(t)
    \in
    \argmin_{a\in\delta\mathbb{Z}} |t-a|.
\end{equation}
The rounding error introduced at this step is
\begin{equation}
    \alpha_i\defeq \vecq_i-\vecw_i^{(i)}.
\end{equation}
Thus, before any compensation is applied to obtain $\vecw^{(i+1)}$, the new contribution to the error vector $\vecq-\vecw$ is $\alpha_i\stdbasis{i}$.

With the coordinates $1,\ldots,i-1$  already  quantized, OPTQ compensates for the error $\alpha_i\stdbasis{i}$ using only the remaining coordinates $j>i$.  If the tail compensation is written as $\alpha_i\mathbf{z}$ with $\mathbf{z}\in\R^{N-i}$, then the corresponding  error to be minimized is
\(
    \left\|
        \matX\left(
            \alpha_i\stdbasis{i}
            +
            \alpha_i
            \begin{pmatrix}
                \mathbf{0}_i\\
                \mathbf{z}
            \end{pmatrix}
        \right)
    \right\|^2.
\)
After factoring out $\alpha_i^2$, the optimization problem becomes
\begin{equation}
    \min_{\mathbf{z}\in\R^{N-i}}
    \left\|
        \matX_i+\matX_{>i}\mathbf{z}
    \right\|^2.
\end{equation}
Assuming for the moment that $\matX$ has full column rank with $m>N$, the least-squares solution is
\begin{equation}
    \mathbf{z}=-\matX_{>i}^{\dagger}\matX_i.
\end{equation}
This gives the update direction
\begin{equation}
    (\vup{i})_{<i}=0,\qquad
    (\vup{i})_i=1,\qquad
    (\vup{i})_{>i}=-\matX_{>i}^{\dagger}\matX_i.
    \label{eq:vdef_pseudoinverse}
\end{equation}
Equivalently, $\vup{i}$ is characterized by the variational problem
\begin{align}
    \vup{i}\defeq \argmin_{\vecv\in\R^N}\quad
    & \|\matX\vecv\|^2
    \label{eq:vdef_variational}\\
    \text{subject\ to}\quad
    & \vecv_j=0,\qquad j<i,\notag\\
    & \vecv_i=1.\notag
\end{align}
The weight vector is then updated as
\begin{equation}
    \vecw^{(i+1)}\defeq\vecw^{(i)}+\alpha_i\vup{i}.
\end{equation}
Since the coordinates $1,\ldots,i$ are fixed after step $i$, after all $N$ steps we obtain
\begin{equation}
    \vecr\defeq \vecq-\vecw
    =
    \sum_{i=1}^N \alpha_i\vup{i}.
\end{equation}
This representation of the OPTQ quantization error will be useful throughout the analysis.

We now introduce the regularization parameter $\lambda$.  The procedure described above corresponds to OPTQ with $\lambda=0$.  For $\lambda>0$, OPTQ is equivalent to running the same procedure with an augmented calibration matrix $\tilde{\matX}$ in place of $\matX$, where
\begin{equation}
    \tilde{\matX}
    \defeq
    \begin{pmatrix}
        \matX\\
        \sqrt{\lambda}\matI
    \end{pmatrix}.
\end{equation}
Note that $\tilde{\matX}\in\R^{(m+N)\times N}$ is always full column rank.  Therefore, with $\lambda>0$, the update directions are well-defined for any calibration matrix $\matX$.  In this case, the preceding formulas are interpreted with $\matX$ replaced by $\tilde{\matX}$. In particular,
\begin{equation}
    (\vup{i})_{<i}=0,\qquad
    (\vup{i})_i=1,\qquad
    (\vup{i})_{>i}
    =
    -\tilde{\matX}_{>i}^{\dagger}\tilde{\matX}_i.
\end{equation}

Let us now  elaborate on how the coefficients $\alpha_i$ are chosen.  For deterministic OPTQ, 
the $i$-th coordinate $\wup{i}_i$ is rounded to the nearest multiple of $\delta$, and the corresponding residual is
\begin{equation}
    \alpha_i
    \defeq
    \delta
    \left\lfloor
        \frac{\wup{i}_i}{\delta}+\frac{1}{2}
    \right\rfloor
    -
    \wup{i}_i.
\end{equation}
This choice gives $\vecq_i=\wup{i}_i+\alpha_i$ and  $|\alpha_i|\leq \delta/2$.

We will also consider a stochastic variant, studied by \citet{NEURIPS2023_0df38cd1,zhang2025provable}, which differs only in the resulting $\alpha_i$. 
Let $p_i \defeq \wup{i}_i / \delta - \lfloor\wup{i}_i/\delta\rfloor$ be the remainder when $\wup{i}_i$ is divided by $\delta$.
Stochastic OPTQ sets
\begin{equation}
\label{eq:alpha_randomness}
    \alpha_i
    \defeq
    \begin{cases}
        -p_i\delta &\text{with probability $1-p_i$,}\\
        (1-p_i)\delta &\text{with probability $p_i$.}
    \end{cases}
\end{equation}
With this choice, $\ev[\alpha_i]=0$ and $|\alpha_i|\leq\delta$ almost surely.  We will refer to the deterministic and stochastic variants collectively as OPTQ when the distinction is not important.

In what follows,  the quantization error $\vecr$,  update directions $\vup{i}$, and  coefficients $\alpha_i$ will be central to our analysis.  These quantities depend on the calibration matrix and on the regularization parameter through the augmented matrix $\tilde{\matX}$.  When this dependence is relevant, we will indicate it explicitly.

Finally, we note a geometric consequence of the variational characterization.  Since $\vup{i}$ minimizes $\|\tilde{\matX}\vecv\|^2$ among vectors whose first nonzero coordinate is a unit entry in position $i$, the vector $\tilde{\matX}\vup{i}$ is the projection of $\tilde{\matX}_i$ onto the orthogonal complement of the span of the columns of $\tilde{\matX}_{>i}$.  We denote this projection operator by $\proj{i}$, so that $\proj{i}\tilde{\matX}_i = \tilde{\matX}\vup{i}$.

\section{Main results}

We begin by isolating the step whereby we use an empirical result to obtain a population result from the particular quantization algorithm. The following theorem shows that, under regularization, the regularized empirical error controls the corresponding population error uniformly over all error vectors, including those that depend on the calibration data.

\begin{restatable}{theorem}{generalizationboundgeneral}
\label{thm:generalization_bd_general}
Suppose we sample $m$ i.i.d.\ random samples $\xsamp{i}$ of a random vector $\vecx\in\R^N$ satisfying $\|\vecx\|_2\leq R$ almost surely. Denote
the $m\times N$ matrix whose $i$-th row is ${\xsamp{i}}^\transpose$ by 
\begin{equation}
    \matX \defeq
    \begin{pmatrix}
        \xsamp{1} & \cdots & \xsamp{m}
    \end{pmatrix}^{\transpose}
\end{equation}
 and let
\(
    \matSigma \defeq \ev_{\vecx}[\vecx\vecx^\transpose].
\)
Fix $\epsilon\in(0,1)$ and assume that
$\lambda > 2R^2\log(N/\epsilon)$.
Then, with probability at least $1-\epsilon$ over the randomness of the sample $\matX$, the following holds simultaneously for all $\vecr\in\R^N$:
\begin{equation}
    \ev_{\vecx}\left[|\vecx^\transpose\vecr|^2\right]
    \leq
    \left(
    1-\sqrt{\frac{2R^2\log(N/\epsilon)}{\lambda}}
    \right)^{-1}
    \frac{1}{m}
    \left(
        \|\matX\vecr\|^2+\lambda\|\vecr\|^2
    \right).
    \label{eq:general_empirical_to_population}%
\end{equation}
In particular, the conclusion applies to any possibly data-dependent vector $\vecr=\vecr(\matX)$.
\end{restatable}

The preceding theorem is independent of the particular quantization algorithm. It shows that any bound on the regularized empirical error $\|\tilde{\matX}(\vecq-\vecw)\|^2$ immediately yields a corresponding population error bound. We next illustrate the usefulness of this result for OPTQ by invoking OPTQ's bound for the augmented calibration matrix $\tilde{\matX}$ \citep{zhang2025provable}.
Indeed, when applied to OPTQ, with $\vecr = \vecq-\vecw$, the term $\|\tilde{\matX}\vecr\|^2/m$ that appears in the upper bound represents the average ``training error'' of OPTQ with an additional regularization term. That is,
\begin{equation}
    \frac{1}{m}\|\matX\vecr\|^2 +\frac{\lambda}{m}\|\vecr\|^2 = \left(\frac{1}{m}\sum_{i=1}^N| {\xsamp{i}}^\transpose \vecw - {\xsamp{i}}^\transpose \vecq|^2\right) + \frac{\lambda}{m}\|\vecw-\vecq\|^2.
\end{equation}
This term is then multiplied by a factor which can be controlled by the choice of $\lambda$. \Cref{thm:generalization_bd_train}, connecting this bound to the geometry of the calibration data, 
follows immediately by applying Eq.~(3.5) of \citep{zhang2025provable}. 

\begin{restatable}{corollary}{generalizationboundtrain}
\label{thm:generalization_bd_train}
Suppose the hypotheses of \cref{thm:generalization_bd_general} hold. Let $\vecq$ be the output of OPTQ applied to $\vecw$ using the augmented calibration matrix
\(
    \tilde{\matX}
    =
    \begin{pmatrix}
        \matX\\
        \sqrt{\lambda}\matI
    \end{pmatrix},
\)
and let $\vecr\defeq \vecq-\vecw$. Then, with probability at least $1-\epsilon$ over the randomness of the sample $\matX$,
\begin{equation}
    \ev_{\vecx}\left[
        |\vecx^\transpose\vecw-\vecx^\transpose\vecq|^2
    \right]
    \leq
    \frac{C\delta^2}{m}
    \left(
    1-\sqrt{\frac{2R^2\log(N/\epsilon)}{\lambda}}
    \right)^{-1}
    \sum_{i=1}^N
    \|\proj{i}\tilde{\matX}_i\|^2.
    \label{eq:expected_squared_quant_error_full}
\end{equation}
Here $C=1/4$ for deterministic OPTQ and $C=1$ for stochastic OPTQ.
\end{restatable}

The proofs of \cref{thm:generalization_bd_general,thm:generalization_bd_train} are in \Cref{sec:generalization_bd_proof}.  These results demonstrate an important role of the regularization parameter.  In the absence of the term $\lambda\|\vecr\|^2$, the empirical quadratic form $\|\matX\vecr\|^2$ need not control $\vecr^\transpose\matSigma\vecr$ uniformly over data-dependent error vectors $\vecr$, since directions that are poorly represented (or absent) in the calibration sample may still have nonzero population variance.  The regularization term makes every direction visible to the empirical objective, at the cost of introducing the additional contribution $\lambda\|\vecr\|^2$ into the bound.

We now proceed to state the second major result of this paper, which bounds arbitrary positive semi-definite quadratic forms of $\vecr$ when using stochastic OPTQ.
\begin{restatable}{theorem}{psdbound}\label{thm:psd_bound}
Let $\matA\in\R^{(m+N)\times (m+N)}$ be an arbitrary positive semi-definite matrix.  For stochastic OPTQ, we can bound
\begin{equation}
\ev_\alpha [ (\tilde{\matX}\vecr)^\transpose \matA \tilde{\matX} \vecr] \leq \frac{\delta^2}{4}\left(\max_i \|\proj{i}\tilde{\matX}_i\|^2\right)\tr \matA,\label{eq:psd_bound_ev}
\end{equation}
where $\ev_\alpha[\cdot]$ denotes the expectation is taken only over the randomness in \cref{eq:alpha_randomness} associated with choosing $\alpha$, with the sample matrix $\matX$ fixed.

Furthermore, with probability at least $1-\epsilon$ over the randomness in $\alpha_i$, we have
\begin{equation}
    (\tilde{\matX}\vecr)^\transpose \matA \tilde{\matX} \vecr \leq 2\delta^2\log(2/\epsilon)\left(\max_i \|\proj{i}\tilde{\matX}_i\|^2\right)\tr \matA.\label{eq:psd_bound_hp}
\end{equation}
\end{restatable}
This theorem is proven in \cref{sec:psd_bound_proof}.  \Cref{thm:psd_bound} holds very broadly and has a variety of applications.  An important application for this work is as follows.
\begin{corollary}\label{cor:generalization_bd_as}
    Let $\vecz \in \R^N$ be a random vector with finite second moment.  Then, when stochastic OPTQ with $\lambda > 0$ is run with a data matrix $\matX$ we have
    \begin{equation}
        \ev_\alpha \ev_{\vecz} [|\vecz^\transpose \vecw - \vecz^\transpose \vecq|^2] \leq \frac{\delta^2}{4}\left(\frac{\max_i \|\proj{i}\tilde{\matX}_i\|^2}{\lambda}\right)\ev_{\vecz} \left[\|\vecz\|^2\right],
    \end{equation}
    where $\ev_{\vecz}$ denotes the expectation taken over $\vecz$.  With probability at least $1-\epsilon$ over the randomness of $\alpha_i$, we also have
    \begin{equation}
        \ev_{\vecz} [|\vecz^\transpose \vecw - \vecz^\transpose \vecq|^2] \leq 2\delta^2\log(2/\epsilon)\left(\frac{\max_i \|\proj{i}\tilde{\matX}_i\|^2}{\lambda}\right)\ev_{\vecz} \left[\|\vecz\|^2\right].
    \end{equation}
\end{corollary}
\begin{proof}
    Apply \cref{thm:psd_bound} with
    \begin{equation}
        \matA=(\tilde{\matX}^\dagger)^\transpose(\ev_{\vecz} \vecz\vecz^\transpose)\tilde{\matX}^\dagger.
    \end{equation}
    Since $\lambda>0$, it follows that $\tilde{\matX}$ has full column rank, so $\tilde{\matX}^\dagger\tilde{\matX}=\matI$.  Consequently,
    \begin{equation}
        (\tilde{\matX}\vecr)^\transpose\matA(\tilde{\matX}\vecr)=\vecr^\transpose(\ev_{\vecz}\vecz\vecz^\transpose)\vecr=\ev_{\vecz}|\vecz^\transpose\vecr|^2 .
    \end{equation}
    Moreover,
    \begin{equation}
        \tr(\matA)=\tr\left((\ev_{\vecz}\vecz\vecz^\transpose)(\tilde{\matX}^\transpose\tilde{\matX})^{-1}\right)\leq \frac{1}{\lambda}\ev_{\vecz}\|\vecz\|^2,
    \end{equation}
    because $\tilde{\matX}^\transpose\tilde{\matX}=\matX^\transpose\matX+\lambda\matI\succeq \lambda\matI$.
\end{proof}
Note that this bound holds for each fixed calibration matrix $\matX$, and the only probability in the second part of the corollary is the stochastic rounding probability.  It is therefore distribution-agnostic in $\vecz$, even when the distribution of $\vecz$ differs from the distribution that generated $\matX$.   If we instead used memoryless scalar quantization (MSQ), which is often also referred to as round-to-nearest (RTN), we generically would expect
\begin{equation}
    \ev_{\vecz}[|\vecz^\transpose \vecw - \vecz^\transpose\vecq|^2] \gtrsim \delta^2\ev_{\vecz}[\|\vecz\|^2].\label{eq:rtn-generalization-estimate}
\end{equation}
Therefore, using the data-dependent algorithm stochastic OPTQ with sufficiently large choice of $\lambda$ will produce a quantized vector that performs only a bounded factor worse in the worst case compared to a data-independent algorithm.  In \cref{sec:rtn-analysis}, we make the statement in \cref{eq:rtn-generalization-estimate} precise.

Another application of \cref{thm:psd_bound} is the following, which improves a result in \citep[Theorem 4.6]{zhang2025provable} by a factor of $\sqrt{\pi}$.
\begin{corollary}
When stochastic OPTQ with $\lambda > 0$ is run with a data matrix $\matX$, with probability at least $1-\epsilon$ over the randomness of $\alpha_i$ we have
\begin{equation}
    \|\vecw - \vecq\|_\infty \leq \delta\sqrt{2\log(2(m+N)/\epsilon)}\left(\frac{\max_i \|\proj{i}\tilde{\matX}_i\|}{\sqrt{\lambda}}\right).
\end{equation}
\end{corollary}
\begin{proof}
    Apply \cref{eq:psd_bound_hp} of \cref{thm:psd_bound} to the $(m+N)\times (m+N)$ coordinate projectors $\stdbasis{1}{\stdbasis{1}}^\transpose,\ldots,\stdbasis{m}{\stdbasis{m}}^\transpose$ and use a union bound over the failure probabilities.  This shows, with probability at least $1-\epsilon$, we have
    \begin{equation}
    \|\tilde{\matX}\vecw - \tilde{\matX}\vecq\|_\infty \leq \delta\sqrt{2\log(2(m+N)/\epsilon)}\left(\max_i \|\proj{i}\tilde{\matX}_i\|\right).
\end{equation}
The desired result follows by recognizing the last $N$ entries of $\tilde{\matX}\vecw - \tilde{\matX}\vecq$ as $\sqrt{\lambda}(\vecw-\vecq)$.
\end{proof}
The utility of this result is that it bounds above the $\ell^\infty$-norm of $\vecq$ based on $\lambda$, $\matX$, and $\vecw$.  This allows us to use the alphabet in \cref{eq:alphabet_delta_b} and determine an appropriate value of $b$ so that with high probability no clipping will need to be performed.

\section{Choosing \texorpdfstring{$\lambda$}{lambda}}

\subsection{Theory}\label{sec:lambda_theory}

\Cref{thm:generalization_bd_train,cor:generalization_bd_as} bound the generalization error of OPTQ in two different ways, but both bounds feature a dependence on $\lambda$.  Suppose each row of $\matX$ is drawn i.i.d.\ from a distribution with compact support in the ball of radius of $R$.  If the following two inequalities hold
\begin{align}
    \lambda &\gtrsim R^2 \log (N) \label{eq:lambda_lb_R2logN},\\
    \lambda &\gtrsim \max_i \|\matX_i\|^2,
\end{align}
then with high probability we can bound for a random vector $\vecz$
\begin{equation}
        \ev_{\vecz} [|\vecz^\transpose \vecw - \vecz^\transpose \vecq|^2] \lesssim \frac{\|\tilde{\matX}\vecr\|^2}{m}\label{eq:sim_samp_bound}
\end{equation}
whenever $\vecz$ is drawn from the same distribution that each row of $\matX$ was and
\begin{equation}
            \ev_{\vecz} [|\vecz^\transpose \vecw - \vecz^\transpose \vecq|^2] \lesssim \delta^2\ev_{\vecz} \left[\|\vecz\|^2\right].\label{eq:sim_always_bound}
\end{equation}
always.  In general, we won't know the value of $R$ in \cref{eq:lambda_lb_R2logN}.  However, a lower bound for $R$ is the maximum row norm of $\matX$.  Therefore, if we wish to have both bounds \cref{eq:sim_samp_bound,eq:sim_always_bound} we should ensure
\begin{equation}
    \lambda \gtrsim \max\left\{\max_i \|\matX_i\|^2,\max_i\|(\matX^\transpose)_i\|^2\log(N)\right\}.
\end{equation}

We now  turn our attention to \cref{eq:expected_squared_quant_error_full} of \cref{thm:generalization_bd_train}.  This bound depends on $\lambda$ in two ways: through the regularized training error term $\|\tilde{\matX}\vecr\|^2/m = (\|\matX\vecr\|^2+\lambda\|\vecr\|^2)/m$ and the multiplicative factor in front of it.  We consider how to choose $\lambda$ to minimize this bound.  
\begin{remark}
    To minimize the bound in \cref{eq:general_empirical_to_population}, 
    a good choice of $\lambda$ is
        \begin{equation}
        \lambda = \beta\cdot\frac{\|\matX\|_{\textrm{F}}^2}{m^{1/3}\min\{m,N\}^{2/3}},\label{eq:lambda_rec_sampling}
    \end{equation}
    where $\beta$ is a scaling factor to be determined empirically.
\end{remark}
We proceed to justify this remark.  From \cref{thm:generalization_bd_train}, for $\lambda > 2R^2\log(N/\epsilon)$ we have
    \begin{align}
    \ev\left[ |\vecx^\transpose\vecw - \vecx^\transpose\vecq|^2\right]
    &\leq \frac{C\delta^2}{m}\left(1-\sqrt{\frac{2R^2\log (N/\epsilon)}{\lambda}} \right)^{-1}\sum_{i=1}^N \|\tilde{\matX}\vup{i}(\lambda)\|^2.\label{eq:after_zhang_bd}
    \end{align}
    Here $C=1$ for stochastic OPTQ and $C=1/4$ for deterministic OPTQ, corresponding respectively to $|\alpha_i|^2\leq \delta^2$ and $|\alpha_i|^2\leq \delta^2/4$.
    We will attempt to minimize the bound in \cref{eq:after_zhang_bd}.  Since the bound in \cref{eq:after_zhang_bd} is strictly positive, minimizing it is equivalent to minimizing
    \begin{equation}
    f(\lambda) = \log\left(\sum_{i=1}^N \|\tilde{\matX}\vup{i}(\lambda)\|^2\right) - \log\left(1-\sqrt{\frac{2R^2\log(N/\epsilon)}{\lambda}}\right).
    \end{equation}
    For each $i$, the function $g_i: \lambda \mapsto \|\tilde{\matX}\vup{i}(\lambda)\|^2$ is differentiable for $\lambda > 0$, as can be seen by standard formulas for matrix-vector calculus.  However, this function also has the following representation:
    \begin{equation}
    g_i(\lambda) = \min_{\vecv\in V_i}\left[\|\matX\vecv\|^2+\lambda\|\vecv\|^2\right],\label{eq:variation-representation-gi}
    \end{equation}
    where $V_i$ is the affine subspace
  $V_i = \stdbasis{i} + \operatorname{span}\{\stdbasis{j} : i < j \leq N\}$.
    The minimum in \cref{eq:variation-representation-gi} is attained at exactly one point $\vecv=\vup{i}(\lambda)$.  In particular, this means that each function $g_i(\lambda)$ has derivative equal to
    $\|\vup{i}(\lambda)\|^2$.
    We therefore see that
    \begin{equation}f'(\lambda) = \frac{\sum_{i=1}^N \|\vup{i}(\lambda)\|^2}{\sum_{i=1}^N \|\tilde{\matX}\vup{i}(\lambda)\|^2} - \frac{\sqrt{R^2\log(N/\epsilon)}}{\sqrt{2\lambda^3}\left(1-\sqrt{\frac{2R^2\log(N/\epsilon)}{\lambda}}\right)}\,.\label{eq:ub_lambda_deriv}
    \end{equation}
    The first term of \cref{eq:ub_lambda_deriv} is at most $1/\lambda$.  If we let $\lambda = \alpha R^2\log(N/\epsilon)$.  We can then bound
    \begin{align}
        f'(\lambda) &\leq \frac{1}{\lambda}\left(1 - \frac{1}{\sqrt{2\alpha}-2}\right).
    \end{align}
    Thus $f'(\lambda) \leq 0$ whenever $\lambda \leq 4.5R^2\log(N/\epsilon)$.
    
    On the other hand, the first term of \cref{eq:ub_lambda_deriv} is at least $N/(N\lambda+\|\matX\|_{\textrm{F}}^2)$.  Therefore, we can bound
    \begin{align}
        f'(\lambda) &\geq \frac{N\sqrt{2\lambda^3}\left(1-\sqrt{\frac{2R^2\log(N/\epsilon)}{\lambda}}\right) - \sqrt{R^2\log(N/\epsilon)}\left(N\lambda + \|\matX\|_{\textrm{F}}^2\right)}{\left(N\lambda + \|\matX\|_{\textrm{F}}^2\right)\sqrt{2\lambda^3}\left(1-\sqrt{\frac{2R^2\log(N/\epsilon)}{\lambda}}\right)}\\
        &= \frac{N\sqrt{2\lambda^3}-(3N\lambda+\|\matX\|_{\textrm{F}}^2)\sqrt{R^2\log(N/\epsilon)}}{\left(N\lambda + \|\matX\|_{\textrm{F}}^2\right)\sqrt{2\lambda^3}\left(1-\sqrt{\frac{2R^2\log(N/\epsilon)}{\lambda}}\right)}
    \end{align}
    and so $f'(\lambda) \geq 0$ whenever
    \begin{align}\lambda &\geq \left(\sqrt{\frac{R^2\log(N/\epsilon)}{2}}\cdot\frac{3N\lambda+\|\matX\|_{\textrm{F}}^2}{N}\right)^{2/3}.
    \end{align}
    In particular, $f'(\lambda) \geq 0$ whenever
    \begin{align}
    \lambda &\geq \left(\sqrt{2R^2\log(N/\epsilon)}\max\left\{3\lambda, \frac{\|\matX\|_{\textrm{F}}^2}{N}\right\}\right)^{2/3}.
    \end{align}
    This means that the value $\lambda_*$ of $\lambda$ that minimizes $f(\lambda)$ satisfies
    \begin{equation}\label{eq:lambda_*_bd}
        4.5R^2\log(N/\epsilon) \leq \lambda_* \leq \max\left\{18R^2\log(N/\epsilon),\left(\frac{\sqrt{2R^2\log(N/\epsilon)}\|\matX\|_{\textrm{F}}^2}{N}\right)^{2/3}\right\}.
    \end{equation}

    In the regime where each row of $\matX$ is an i.i.d.\ sample from the same distribution as a generic test point, we should choose $\lambda$ satisfying \cref{eq:lambda_*_bd}.  We now proceed to simplify this bound.  Let $K$ be such that $R^2 = K^2 \ev[\|\vecz\|^2]$.  For distributions for which the empirical second moment is representative of its expectation, $\|\matX\|_{\textrm{F}}^2/m \approx \ev\|\vecz\|^2$, and hence $R^2 \approx K^2 \|\matX\|_{\textrm{F}}^2/m$.  This replacement is, of course, heuristic unless accompanied by an explicit concentration assumption on the row norms.  With this definition of $K$, \cref{eq:lambda_*_bd} becomes
    \begin{equation}
    \frac{\|\matX\|_{\textrm{F}}^2}{m} \lesssim_{K,\log} \lambda_* \lesssim_{K,\log} \frac{\|\matX\|_{\textrm{F}}^2}{m^{1/3}\min\{m,N\}^{2/3}},
    \end{equation}
    where $\lesssim_{K,\log}$ hides the dependence on $K$ and $\log$ factors.  In light of \cref{cor:generalization_bd_as}, it is sensible to use the upper bound as a conservative estimate for $\lambda$.  This leads to the scaling in \cref{eq:lambda_rec_sampling}.  While this argument is not precise enough to determine an optimal scaling constant, in the next subsection we will empirically determine a good choice of constant.

\subsection{Experiments}

We now demonstrate how the recommendations from \cref{sec:lambda_theory} work in practice.  To do this, we run the output of OPTQ with varying choices of $\lambda$ for data drawn from three different distributions and show the generalization error of the result.\footnote{See \url{https://github.com/ErinGeorge/OPTQ-generalization} for the code used for these experiments.}

We consider a total of seven different choices of $\lambda$ for OPTQ.  Five of these choices correspond to \cref{eq:lambda_rec_sampling} with different values of $\beta$.  The remaining thwo are recommendations previously made in the literature.  These are $\lambda_1 = 0.01 \|\matX\|_{\textrm{F}}^2 / N$ from \citet{frantar2023optq} and $\lambda_2 = 10^{-3} \cdot\|\matX\|_{\textrm{op}}$ from \citet{zhang2025qronos}.\footnote{The recommendation $\lambda_2$ is for Qronos, but the effect of $\lambda$ on OPTQ is identical.}

In all experiments, we sample the rows of the $m \times N$ data matrix $\matX$ i.i.d.\ from the same distribution as a generic test point $\vecz$ and vary the number of rows $m$.  We simplify the measures of the generalization error by reporting the value of
\begin{equation}
    E = \sum_{i=1}^N \ev[|\vecz^\transpose \vup{i}|^2]\label{eq:test_error_formula}
\end{equation}
as the test error.  For stochastic OPTQ, this expression is precisely $\ev[|\vecz^\transpose \vecr|^2]$ where the contribution from the $\alpha_i$ is ignored (see \cref{eq:ev_alpha_bd_initial,eq:hp_alpha_bd_initial} from the proof of~\cref{thm:psd_bound}).  The contribution of each $\alpha_i$ depends on both the randomness of stochastic OPTQ and the individual quantization error for each coordinate.  Up to constant factors of $\delta$, \cref{eq:test_error_formula} can be regarded as both a worst-case for the generalization error and an average-case for the generalization error (over the randomness of $\alpha_i$ and the initial position of $\vecw$).   We approximate the expectations in \cref{eq:test_error_formula} with $1000$ samples of $\vecz$.  For all the plots in this section, the mean and standard deviation over $30$ trials (each an independent sample of $\matX$) is shown.

\paragraph{Uniform Hadamard} The first data distribution we consider is the uniform distribution on rows from Sylvester's construction of a $64\times 64$ Hadamard matrix.   In the full data distribution, each coordinate is independent of the others.  However, in a low sampling regime, many spurious correlations between the columns of $\matX$ can appear.  As such, OPTQ performs poorly when improperly regularized. The results of these experiments are shown in \cref{fig:uniform_hadamard}.

\begin{figure}[t!]
    \includegraphics[width=0.35\textwidth]{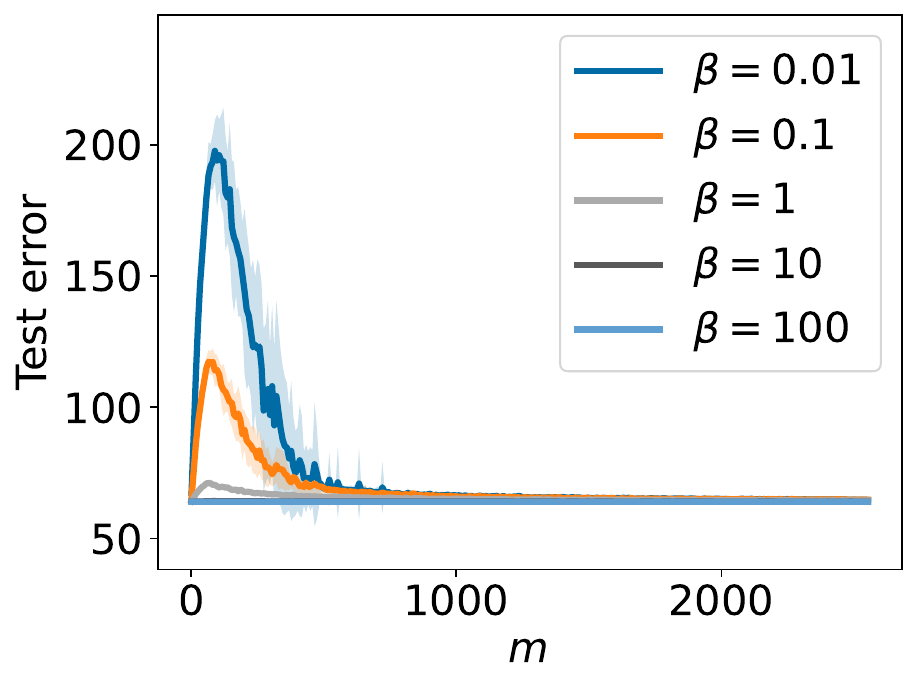}
    \includegraphics[width=0.35\textwidth]{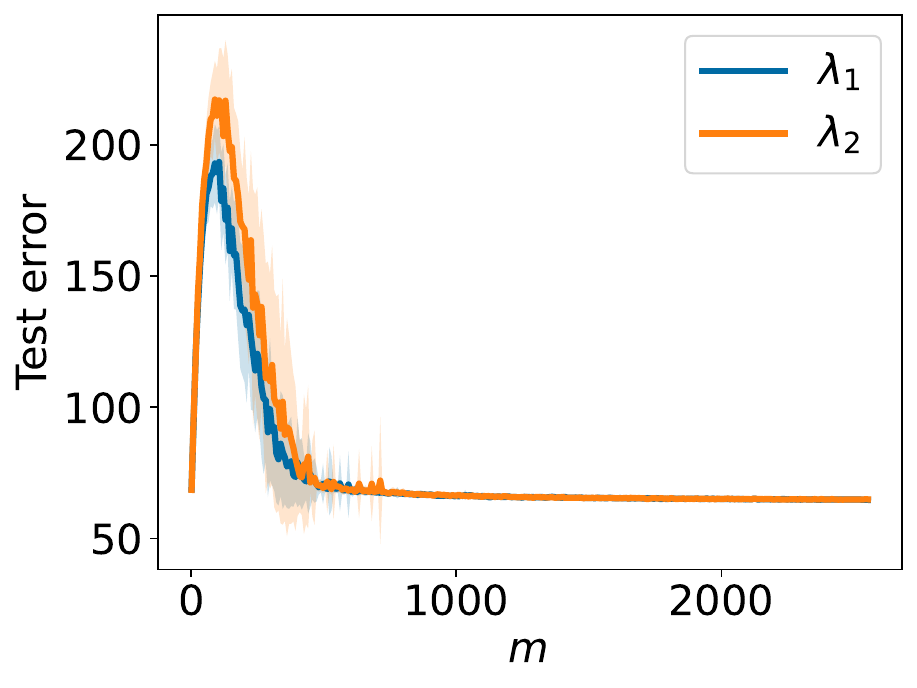}
    \centering
    \caption{Test error versus $m$ for various choices of $\lambda$ when OPTQ is run with samples from the uniform Hadamard distribution.}
    \label{fig:uniform_hadamard}
\end{figure}

\paragraph{Non-uniform Hadamard} The second data distribution we consider again draws rows from a $64\times 64$ Hadamard matrix, but with a non-uniform distribution on the rows.  The probability of a row $i$ being drawn is proportional to $1/i$.  Similar to the uniform Hadamard distribution, this data distribution results in many spurious correlations appearing in a low sampling regime.  However, the true data distribution does have a low-rank structure exploitable by OPTQ to reduce the error.  The results of these experiments are shown in \cref{fig:nonuniform_hadamard}.

\begin{figure}
    \includegraphics[width=0.35\textwidth]{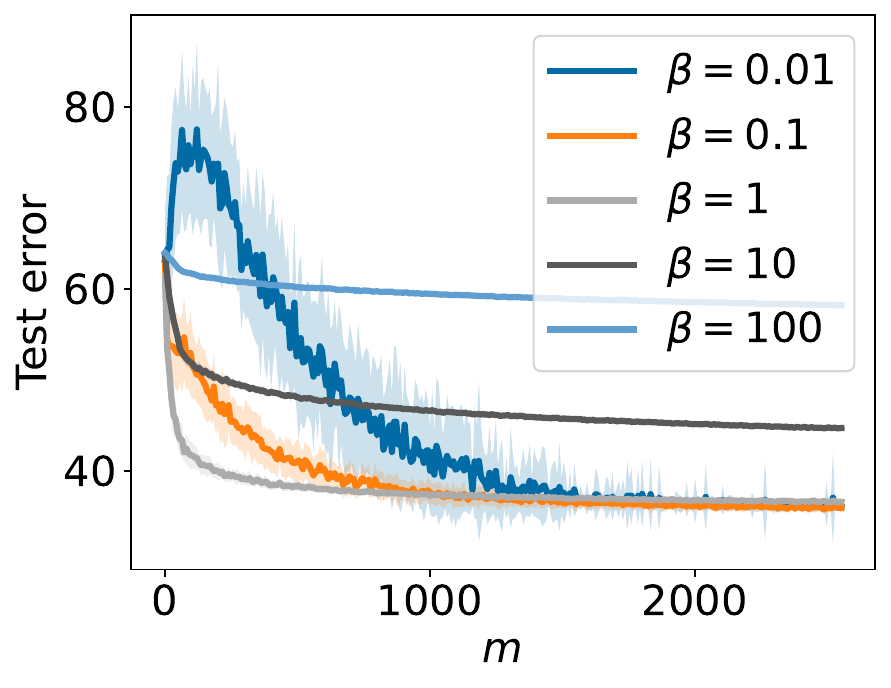}
    \includegraphics[width=0.35\textwidth]{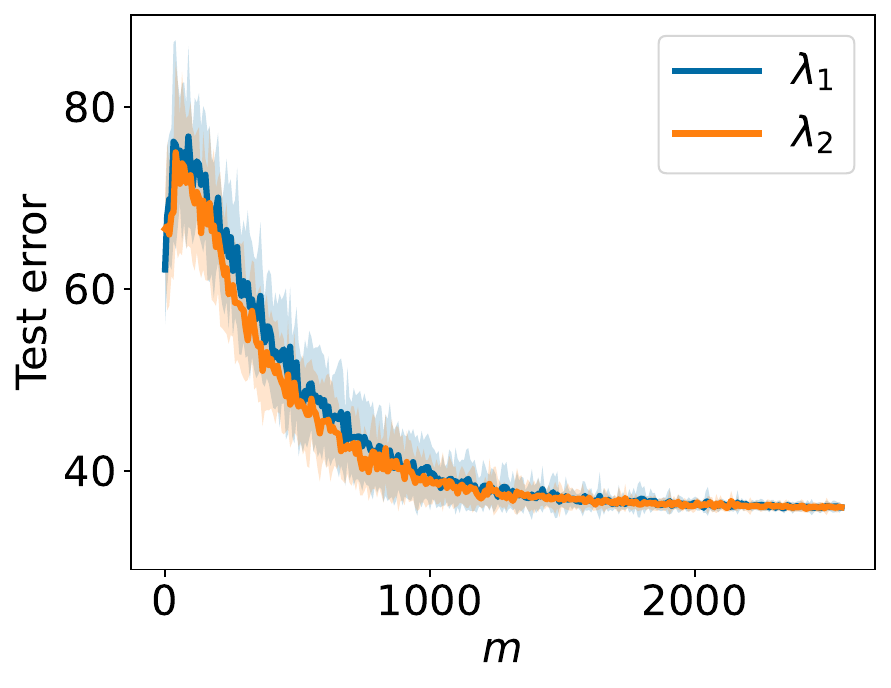}
    \centering
    \caption{Test error versus $m$ for various choices of $\lambda$ when OPTQ is run with samples from the non-uniform Hadamard distribution.}
    \label{fig:nonuniform_hadamard}
\end{figure}

\paragraph{ReLU neural network}  The last data distribution we consider is a model of a ReLU neural network.  First, we fix a $64\times 16$ matrix $\matW$ where each coordinate is sampled i.i.d.\ from a standard normal distribution.  We generate a random vector $\vecz \in \R^{64}$ by first sampling a random vector $\vecn \in \R^{16}$ with i.i.d.\ standard normal coordinates and then applying the ReLU function\footnote{The ReLU function $\phi$ is defined as $\phi(x) = \max\{x,0\}$.} coordinate-wise to $\matW\vecn$.  This is a model of what the data distribution of the output of a single neural network layer may look like.  Here, the low-rank structure comes from the overparameterization of the neural network layer and the positivity enforced by the ReLU function.  The results of these experiments are in \cref{fig:relu_nn}.

\begin{figure}[t!]
    \includegraphics[width=0.35\textwidth]{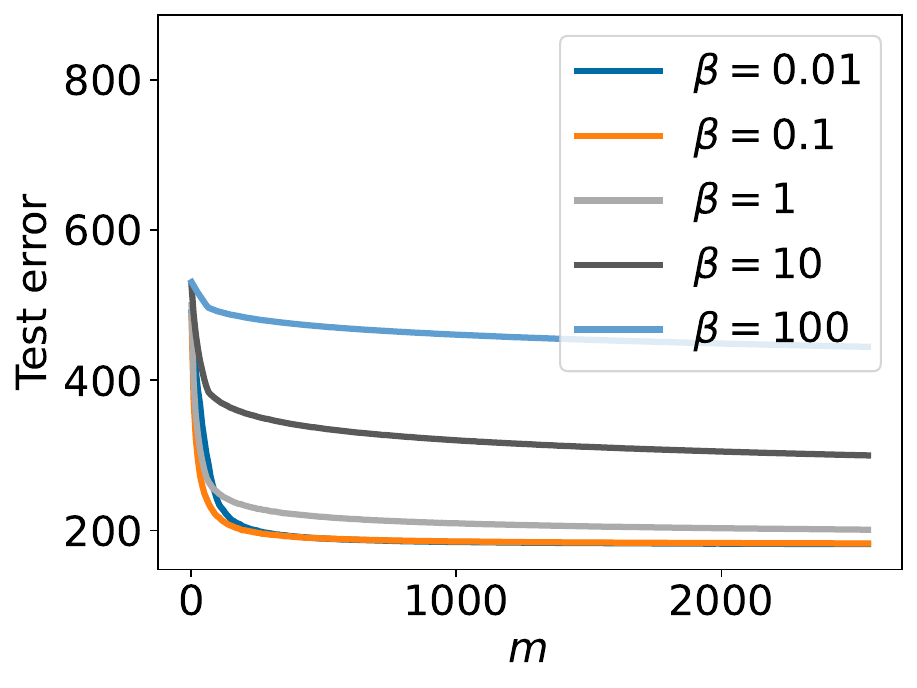}
    \includegraphics[width=0.35\textwidth]{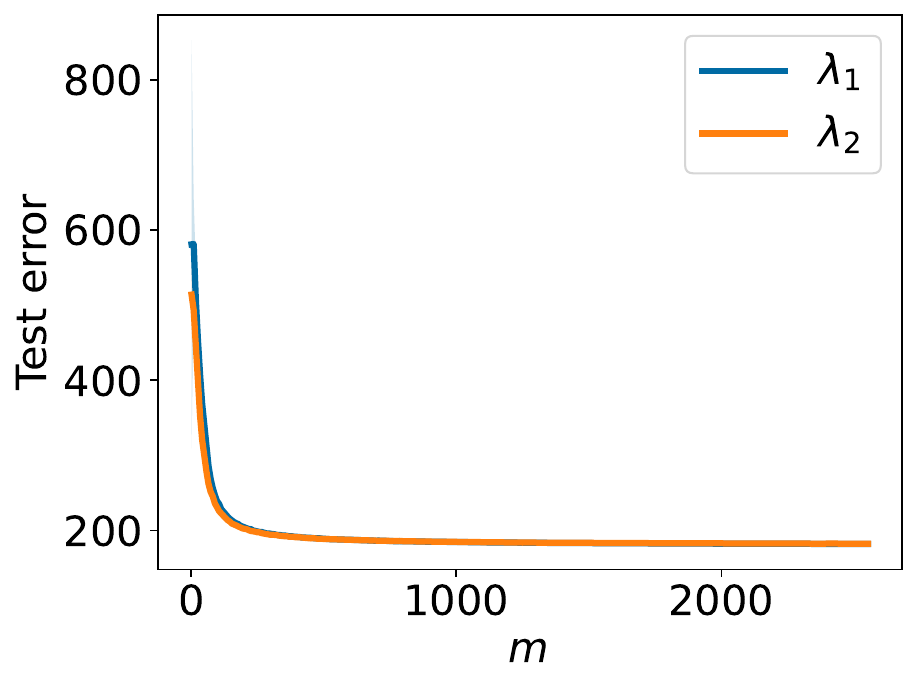}
    \centering
    \caption{Test error versus $m$ for various choices of $\lambda$ when OPTQ is run with samples from the ReLU neural network distribution.}
    \label{fig:relu_nn}
\end{figure}

Across \cref{fig:uniform_hadamard,fig:nonuniform_hadamard,fig:relu_nn}, a clear trend emerges.  The existing recommendations for $\lambda$ both do well in data distributions with low-rank structure in a well-sampled regime, but are susceptible to producing large errors in a low sampling regime.  However, choosing $\lambda$ according to \cref{eq:lambda_rec_sampling} can result in the same error being achieved in the well-sampled regime without the large error present in the low sampling regime.  Across all the experiments, choosing $\beta=1$ is able to substantially reduce the error in \cref{fig:uniform_hadamard,fig:nonuniform_hadamard} with low $m$ while still achieving small values for the error in \cref{fig:nonuniform_hadamard,fig:relu_nn} for large $m$.

\section{Conclusion}

In this work, we bound the generalization error of OPTQ in two different settings: when the input data is drawn i.i.d.\ from the same distribution as a test point and when randomized rounding is used as the quantization procedure for OPTQ.  Using these results, we derive a formula for choosing the quantization parameter $\lambda$ that outperforms previous recommendations in a low-sampling regime while still achieving similarly good generalization performance in a high-sampling regime.

\subsubsection*{Acknowledgments}

E.G.\ was supported by NSF grant DMS-2503574.
R.S.\ was partially supported by grant DMS-2410717.

\subsubsection*{AI disclosure}

In this work, we used generative AI tools to find references in the literature and to edit the paper.  All AI-provided references and edits were then reviewed by the authors.  We did not use generative AI tools for any other task.

\bibliographystyle{plainnat}
\bibliography{refs}

\appendix

\crefalias{section}{appendix}

\section{Proof of \texorpdfstring{\cref{thm:generalization_bd_general,thm:generalization_bd_train}}{Theorem 3.1 and Corollary 3.2}}\label{sec:generalization_bd_proof} To prove \cref{thm:generalization_bd_general}, we will use the following result.

\begin{theorem}[{Matrix Chernoff inequality \citep[Remark 5.3]{tropp_user-friendly_2012}}]\label{thm:mat_chernoff_ineq}
    Consider a finite sequence $\{\matA^{(k)}\}$ of independent, random, self-adjoint matrices with dimension $d$. Assume that each random matrix is positive semidefinite and satisfies $\|\matA^{(k)}\|_{\textrm{op}} \leq R$ almost surely.
    
Let $\sigma_{\min}(\cdot)$ be the smallest eigenvalue of its matrix argument and define
\begin{equation}
\mu_{\min} \defeq \sigma_{\min}\left(\sum_k \ev[\matA^{(k)}]\right).
\end{equation}
Then for $t \in [0,1]$,
\begin{align}
\prob\left[\sigma_{\min}\left(\sum_k \matA^{(k)}\right) \leq (1-t)\mu_{\min}\right] &\leq d \exp\left(-\frac{t^2\mu_{\min}}{2R}\right).
\end{align}
\end{theorem}
We now proceed to prove \cref{thm:generalization_bd_general}.
\generalizationboundgeneral* \begin{proof} Let $\matSigma \defeq \ev_{\vecx}[\vecx\vecx^\transpose]. $ We first note that, for any $\vecr\in\R^N$, \begin{equation} \ev_{\vecx}\left[|\vecx^\transpose\vecr|^2\right] = \vecr^\transpose\matSigma\vecr. \label{eq:ev_as_cov_qform_general} \end{equation} Let 
\( \matP \defeq \left(\matSigma+(\lambda/m)\matI\right)^{-1/2} \)
 and let $K=\lceil \lambda/R^2\rceil$. Consider 
 \begin{equation}
  \matS \defeq \matP \frac{1}{m} \left(\matX^\transpose\matX+\lambda\matI\right) \matP. 
 \end{equation} 
  Then $\ev[\matS]=\matI$. We write $\matS$ as
  \begin{equation}
   \matS = \left( \sum_{j=1}^K \frac{\lambda}{Km}\matP^2 \right) + \left( \sum_{i=1}^m \frac{1}{m} \matP\xsamp{i}{\xsamp{i}}^\transpose\matP \right). 
 \end{equation}
Each summand has operator norm at most $R^2/\lambda$.
Indeed, since
\(
    \matSigma+(\lambda/m)\matI \succeq (\lambda/m)\matI,
\)
we have
\(
    \|\matP\|_{\mathrm{op}}^2
    =
    \left\|
    \left(\matSigma+(\lambda/m)\matI\right)^{-1/2}
    \right\|_{\mathrm{op}}^2
    \leq
    \frac{m}{\lambda}.
\)
Therefore, using $\|\xsamp{i}\|_2\leq R$,
\begin{equation}
    \left\|
    \frac{1}{m}
    \matP\xsamp{i}{\xsamp{i}}^\transpose\matP
    \right\|_{\mathrm{op}}
    =
    \frac{1}{m}\|\matP\xsamp{i}\|_2^2
    \leq
    \frac{R^2}{\lambda}.
\end{equation}
Similarly, for the artificial regularization summands,
\(
    \left\|
    \frac{\lambda}{Km}\matP^2
    \right\|_{\mathrm{op}}
    \leq
    \frac{1}{K}
    \leq
    \frac{R^2}{\lambda},
\)
where the last inequality follows from $K=\lceil \lambda/R^2\rceil$.
Therefore, by \cref{thm:mat_chernoff_ineq}, 
\begin{equation}
 \prob\left[ \sigma_{\min}(\matS)\leq 1-t \right] \leq N\exp\left(-\frac{t^2\lambda}{2R^2}\right). 
\end{equation} 
Taking \( t=\sqrt{\frac{2R^2\log(N/\epsilon)}{\lambda}}, \)
 we obtain that, with probability at least $1-\epsilon$, \begin{equation} \sigma_{\min}(\matS) > 1-\sqrt{\frac{2R^2\log(N/\epsilon)}{\lambda}}. \label{eq:min_cov_eval_general} \end{equation} On this event, \begin{align} \sup_{\vecu\in\R^N} \frac{\vecu^\transpose\matSigma\vecu} {(1/m)(\|\matX\vecu\|^2+\lambda\|\vecu\|^2)} &= \sup_{\vecu\in\R^N} \frac{\vecu^\transpose(\matP\matSigma\matP)\vecu} {\vecu^\transpose\matS\vecu} \\ &\leq \left( 1-\sqrt{\frac{2R^2\log(N/\epsilon)}{\lambda}} \right)^{-1}, \end{align} where we used $\|\matP\matSigma\matP\|_{\mathrm{op}}\leq 1$. Combining this estimate with \cref{eq:ev_as_cov_qform_general} proves the result. \end{proof} \generalizationboundtrain* \begin{proof} Apply \cref{thm:generalization_bd_general} to the OPTQ error vector \begin{equation} \vecr=\vecq-\vecw. \end{equation} Then, to obtain \cref{eq:expected_squared_quant_error_full}, we use the OPTQ identity \citep[Eq. (3.5)]{zhang2025provable} \begin{equation} \|\tilde{\matX}\vecr\|^2 = \sum_{i=1}^N |\alpha_i|^2 \|\proj{i}\tilde{\matX}_i\|^2. \label{eq:zhang_train_err_bd} \end{equation} For deterministic OPTQ, $|\alpha_i|\leq \delta/2$, and for stochastic OPTQ, $|\alpha_i|\leq\delta$ almost surely. Hence \begin{equation} |\alpha_i|^2 \leq C\delta^2, \end{equation} where $C=1/4$ for deterministic OPTQ and $C=1$ for stochastic OPTQ. Substituting this into \cref{eq:zhang_train_err_bd} and applying \cref{eq:general_empirical_to_population} 
 completes the proof. \end{proof}

\section{Proof of \texorpdfstring{\cref{thm:psd_bound}}{Theorem 3.2}}\label{sec:psd_bound_proof}

We first will establish the following lemma.

\begin{lemma}\label{lem:vi_A_vi}
Let $\matA\in\R^{(m+N)\times(m+N)}$ be a positive semi-definite matrix.  If we run OPTQ, then
\begin{equation}
    \sum_{i=1}^N (\tilde{\matX}\vup{i})^\transpose \matA \tilde{\matX}\vup{i} \leq \left(\max_{i} \|\proj{i}\tilde{\matX}_i\|^2\right)\tr\matA.
\end{equation}
\end{lemma}

\begin{proof}
Recall that by \cref{eq:vdef_pseudoinverse}, each vector $\tilde{\matX}\vup{i}$ is $\proj{i}\tilde{\matX}_i$, the projection of $\tilde{\matX}_i$ onto the orthogonal complement of the range of $\tilde{\matX}_{>i}$.  Therefore, $\tilde{\matX}\vup{1},\ldots,\tilde{\matX}\vup{N}$ is a collection of orthogonal vectors.  Define
\begin{equation}
    \vecui{i} = \frac{\tilde{\matX}\vup{i}}{\|\tilde{\matX}\vup{i}\|}
\end{equation}
for $i = 1,\ldots,N$ and let
\begin{equation}
    \matU = \begin{pmatrix} \vecui{1} & \cdots & \vecui{N}\end{pmatrix}.
\end{equation}
We bound
\begin{equation}
    \sum_{i=1}^N (\tilde{\matX}\vup{i})^\transpose \matA \tilde{\matX}\vup{i} \leq \left(\max_{i} \|\proj{i}\tilde{\matX}_i\|^2\right)\sum_{i=1}^N {\vecui{i}}^\transpose \matA \vecui{i} 
\end{equation}
and compute
\begin{equation}
   \sum_{i=1}^N {\vecui{i}}^\transpose \matA \vecui{i}  = \tr(\matU^\transpose \matA\matU) = \tr(\matA).\qedhere
\end{equation}
\end{proof}

We now proceed to prove \cref{thm:psd_bound}.

\psdbound*

\begin{proof}
We first show \cref{eq:psd_bound_ev}.  In the stochastic version of OPTQ, each $\alpha_i$ is a random variable that satisfies
\begin{align}
\ev[\alpha_i \mid \alpha_1,\ldots,\alpha_{i-1}] = 0\text{\quad and \quad}\ev[\alpha_i^2 \mid \alpha_1,\ldots,\alpha_{i-1}] \leq \frac{\delta^2}{4}.\label{eq:alpha_ev}
\end{align}
By the law of total expectation,
\begin{align}
\ev[\alpha_i \alpha_j] &= 0 \text{ if $i\neq j$ and}\\
\ev[\alpha_i^2] &\leq \frac{\delta^2}{4}.
\end{align}
We therefore see that
\begin{align}
    \ev_\alpha [(\tilde{\matX}\vecr)^\transpose\matA(\tilde{\matX}\vecr)] &= \sum_{i=1}^N \sum_{j=1}^N \ev[\alpha_i \alpha_j] (\tilde{\matX}\vup{i})^\transpose \matA \tilde{\matX}\vup{j} \\
    &\leq \sum_{i=1}^N \frac{\delta^2}{4} (\tilde{\matX}\vup{i})^\transpose \matA \tilde{\matX}\vup{i}.\label{eq:ev_alpha_bd_initial}
\end{align}
\Cref{eq:psd_bound_ev} follows by applying \cref{lem:vi_A_vi} to \cref{eq:ev_alpha_bd_initial}.

We now consider \cref{eq:psd_bound_hp}.  From \cref{eq:alpha_ev}, we see that the sequence of vectors
\begin{equation}
\vecui{j} = \sum_{i=1}^j \alpha_i \matA^{1/2} \tilde{\matX}\vup{i}
\end{equation}
is a vector martingale.  From \citep[Theorem 3.5]{pinelis1994optimum}, we can bound
\begin{equation}
\prob\left[\max_j \|\vecui{j}\|^2 \geq t\right] \leq 2\exp\left(-\frac{t}{2\delta^2\sum_{i=1}^N (\tilde{\matX}\vup{i})^\transpose \matA \tilde{\matX}\vup{i}}\right)\,.
\end{equation}
Since $(\tilde{\matX}\vecr)^\transpose \matA \tilde{\matX}\vecr = \|\vecui{N}\|^2$, we have
\begin{equation}
\prob\left[(\tilde{\matX}\vecr)^\transpose\matA\tilde{\matX}\vecr \geq t\right] \leq 2\exp\left(-\frac{t}{2\delta^2\sum_{i=1}^N (\tilde{\matX}\vup{i})^\transpose \matA \tilde{\matX}\vup{i}}\right)\,.
\end{equation}
We set $t = 2\delta^2\log(2/\epsilon)\sum_{i=1}^N (\tilde{\matX}\vup{i})^\transpose \matA \tilde{\matX}\vup{i}$ to see that with probability at least $1-\epsilon$ over the randomness of choosing $\alpha_i$ we have
\begin{equation}
(\tilde{\matX}\vecr)^\transpose\matA(\tilde{\matX}\vecr) \leq 2\delta^2\log(2/\epsilon)\sum_{i=1}^N (\tilde{\matX}\vup{i})^\transpose \matA \tilde{\matX}\vup{i}.\label{eq:hp_alpha_bd_initial}
\end{equation}
We then again apply \cref{lem:vi_A_vi} to \cref{eq:hp_alpha_bd_initial} to conclude \cref{eq:psd_bound_hp}.
\end{proof}

\section{Generalization behavior of round-to-nearest algorithms}\label{sec:rtn-analysis}

To precisely state a lower bound for MSQ, we first precisely define the algorithm.  Similarly to OPTQ, we will consider two variants of MSQ: a deterministic variant and a stochastic variant.  In deterministic MSQ, we obtain a quantized vector $\vecq_{\mathrm{MSQ}}$ by rounding each coordinate of $\vecw$ to the nearest multiple of $\delta$.  That is,
\begin{equation}
(\vecq_{\mathrm{MSQ}})_i \defeq \delta\left\lfloor\frac{\vecw_i}{\delta}+\frac{1}{2}\right\rfloor.
\end{equation}
In stochastic MSQ, we do a similar randomized rounding procedure as stochastic OPTQ to obtain a quantized vector $\vecq_{\textrm{SR}}$.  First, for each coordinate we define
\begin{equation}p_i \defeq \vecw_i - \left\lfloor\frac{\vecw_i}{\delta}\right\rfloor\end{equation}
and randomly choose
\begin{equation}
\alpha_i = \begin{cases}-p_i\delta &\text{with probability $1-p_i$} \\
(1-p_i)\delta &\text{with probability $p_i$.}
\end{cases}
\end{equation}
We then set $(\vecq_{\mathrm{MSQ}})_i = \vecw_i + \alpha_i$.

\begin{proposition}\label{prop:msq_rtn_lower}
Let $\delta > 0$ and $\vecz \in \R^N$ be a vector with finite second moment.   Suppose $\vecw$ is determined by sampling from the uniform distribution on $[-\delta/2, \delta/2]^N$.  Choose either stochastic MSQ or deterministic MSQ and let $\vecq_{\textrm{MSQ}}$ be the quantized vector obtained by applying this algorithm to $\vecw$.  There are absolute constants $c_1, c_2 > 0$ such that
\begin{equation}
\prob_{\vecw,\alpha} \bigg[\mathbb{E}_{\vecz} \left[|\vecz^\transpose\vecw-\vecz^\transpose\vecq_{\mathrm{MSQ}}|^2\right] \geq c_1\delta^2\mathbb{E}_{\vecz}\|\vecz\|^2\bigg] \geq c_2.
\end{equation}
Here, $\mathbb{E}_{\vecz}$ denotes that the expectation is taken only over $\vecz$ and $\prob_{\vecw,\alpha}$ denotes the probability with respect to the random choice of $\vecw$ and the randomness in stochastic MSQ, if chosen.
\end{proposition}

\begin{proof}
Denote $\matSigma = \mathbb{E}[\vecz\vecz^\transpose]$.  
Let $\vecr$ denote $\vecw-\vecq_{\mathrm{MSQ}}(\vecw)$.  With this notation,
\begin{equation}\mathbb{E}_{\vecz} \left[|\vecz^\transpose\vecw-\vecz^\transpose\vecq_{\mathrm{MSQ}}|^2\right] = \vecr^\transpose\matSigma\vecr
\end{equation}
and
\begin{equation}
c\delta^2\mathbb{E}_{\vecz}\|\vecz\|^2 = c\delta^2\tr\matSigma.
\end{equation}
Therefore, we need only to demonstrate
\begin{equation}
\prob_{\vecw,\alpha} \left[\vecr^\transpose\matSigma\vecr \geq c\delta^2\tr\matSigma\right] \geq \frac{1}{2}.
    \label{eq:msq-theorem-restated}
\end{equation}

We now proceed to do so. We see that $\vecr$ is a random vector independent of $\vecz$.  This random vector consists of i.i.d.\ mean-zero coordinates such that each coordinate almost surely lies in the interval $[-\delta, \delta]$.  Consequentially, $\vecr$ is a subgaussian random vector with subgaussian norm bounded above by $C_1\delta$, for some absolute constant $C_1$.  We can now apply the Hanson--Wright inequality~\citep[Theorem 1.1]{rudelson2013hanson-wright} to the quadratic form $\vecr^\transpose\matSigma\vecr$ to conclude
\begin{align}
\prob \left[\left|\vecr^\transpose \matSigma \vecr - \mathbb{E}\left[\vecr^\transpose\matSigma\vecr\right]\right| \geq t\right] \leq 2\exp\left(-C_2\min\left\{\frac{t^2}{C_1^4\delta^4\|\matSigma\|_{\textrm{F}}^2},\frac{t}{C_1^2\delta^2\|\matSigma\|_{\textrm{op}}}\right\}\right),\label{eq:msq-after_H-W_ineq}
\end{align}
where $C_2$ is an absolute constant.  We proceed to compute $\mathbb{E}[\vecr^\transpose\matSigma\vecr]$.  This can be done as follows
\begin{align}
    \mathbb{E}[\vecr^\transpose\matSigma\vecr] &= \mathbb{E}\left[\sum_{i=1}^N\sum_{j=1}^N \vecr_i \vecr_j \matSigma_{ij}\right] \\
    &= \sum_{i=1}^N \mathbb{E}[\vecr_i \vecr_i] \matSigma_{ii} \\
    &= \mathbb{E}[\vecr_1^2] \tr \matSigma, \label{eq:msq-r-sigma-ev}
\end{align}
where in the last line we use that the coordinates of $\vecr$ are i.i.d.\ and mean-zero.

In the case of deterministic round-to-nearest $\vecq_1 = 0$ for all possible values of $\vecw_1$.  Therefore, $\vecr_1 = \vecw_1$ is a uniform variable on $[-\delta/2, \delta]$, so
\begin{align}\mathbb{E}[\vecr_1^2] = \frac{1}{\delta}\int_{-\delta/2}^{\delta/2} \vecw_1^2 d\vecw_1 = \frac{\delta^2}{12}.
\label{eq:msq-coord-2nd-moment-deterministic}
\end{align}

In the case of stochastic round-to-nearest, $\vecr_1 = \alpha_1$ by construction.  This depends on the random variable $p_1$, whose marginal distribution is the uniform distribution on $[0,1]$.  Thus, we can compute
\begin{align}
\mathbb{E}[\vecr_1^2] = \int_{0}^{1} \left(p_1(1-p_1)^2\delta^2 + (1-p_1)p_1^2\delta^2\right)dp_1 = \int_{0}^{1} p_1(1-p_1)\delta^2 dp_1 = \frac{\delta^2}{6}.\label{eq:msq-coord-2nd-moment-stochastic}
\end{align}
From \cref{eq:msq-r-sigma-ev,eq:msq-coord-2nd-moment-deterministic,eq:msq-coord-2nd-moment-stochastic}, we obtain the bound
\begin{align}
    \mathbb{E}[\vecr^\transpose\matSigma\vecr] \geq \frac{\delta^2}{12}\tr \matSigma.\label{eq:msq-r-sigma-ev-lb}
\end{align}
We now combine \cref{eq:msq-after_H-W_ineq,eq:msq-r-sigma-ev-lb} to obtain
\begin{align}
    \prob \left[\vecr^\transpose \matSigma \vecr \geq \left(\frac{1}{6}+s\right)\delta^2\tr\matSigma\right] \leq 2\exp\left(-C_2\min\left\{\frac{s^2(\tr\matSigma)^2}{C_1^4\|\matSigma\|_{\textrm{F}}^2},\frac{s\tr\matSigma}{C_1^2\|\matSigma\|_{\textrm{op}}}\right\}\right),
\end{align}
Because $\matSigma$ is a positive semi-definite matrix, we have
\begin{equation}
    \|\matA\|_{\textrm{op}} \leq \|\matA\|_{\textrm{F}} \leq \tr \matA.
\end{equation}
Therefore,
\begin{align}
    \prob \left[\vecr^\transpose \matSigma \vecr \geq \left(\frac{1}{6}+s\right)\delta^2\tr\matSigma\right] \leq 2\exp\left(-C_2\min\left\{\left(\frac{s}{C_1^2}\right)^2,\frac{s}{C_1}\right\}\right).\label{eq:msq-rSigmar-ub}
\end{align}
As $\vecr^\transpose\matSigma\vecr$ is non-negative, \cref{eq:msq-rSigmar-ub} is enough to conclude that $\vecr^\transpose\matSigma\vecr$ is subexponential with subexponential norm bounded by $C_3 \delta^2 \tr \matSigma$ for an absolute constant $C_3$.  This yields the following bound for its second moment~\citep[Proposition 2.8.1]{vershynin_hdp}
\begin{equation}
    \mathbb{E}\left[(\vecr^\transpose\matSigma\vecr)^2\right] \leq C_4 (\delta^2\tr\matSigma)^2
\end{equation}
for another absolute constant $C_4$.  We now apply the Paley-Zygmund inequality to obtain
\begin{equation}\prob\left[\vecr^\transpose\matSigma\vecr \geq \frac{\delta^2}{24}\tr\matSigma\right]\geq \frac{1}{576 C_4}.\qedhere
\end{equation}
\end{proof}

\end{document}